\def\ArxivVersion{1}
\documentclass[journal]{IEEEtran}
\pdftrailerid{}
\usepackage{amsmath}
\usepackage{amssymb}
\usepackage{amsthm}
\usepackage{algorithm}
\usepackage{algpseudocode}
\usepackage{array}
\usepackage{booktabs}
\usepackage{flushend}
\usepackage[nocompress]{cite}
\usepackage{graphicx}
\usepackage{multirow}
\usepackage{needspace}
\usepackage{placeins}
\usepackage{tikz}
\usetikzlibrary{arrows.meta,calc,fit,positioning}
\usepackage{pgfplots}
\usepgfplotslibrary{fillbetween,groupplots}
\pgfplotsset{compat=1.18}
\usepackage{url}

\newtheorem{proposition}{Proposition}

\graphicspath{{figures/selected/}{figures/generated/}}

\newcommand{\JournalTitle}{Latency-Aware Client Assignment for Parallel Split Learning With Global Sampling}

\newcommand{\ManuscriptAuthors}{Mohammad~Kohankhaki,
        Valentin~Rentschler,
        and~Anke~Schmeink}
\newcommand{\AuthorsPlain}{Mohammad Kohankhaki, Valentin Rentschler, Anke Schmeink}
\newcommand{\AuthorInformation}{%
\thanks{The authors are with the Chair of Information Theory and Data Analytics,
RWTH Aachen University, Kopernikusstr. 16, 52074 Aachen, Germany.}%
\thanks{Corresponding author: Mohammad Kohankhaki (e-mail:
mohammad.kohankhaki@inda.rwth-aachen.de).}%
\thanks{Mohammad Kohankhaki and Anke Schmeink are also affiliated with the Cluster of Excellence CARE,
TU Dresden and RWTH Aachen, Germany.}}

\ifdefined\ArxivVersion
  \let\CurrentTitle\JournalTitle
\else
  \let\CurrentTitle\JournalTitle
\fi
\title{\CurrentTitle}
\ifdefined\AnonymousVersion
  \author{Anonymous Authors}
\else
  \author{\ManuscriptAuthors\AuthorInformation\thanks{This work has been submitted to the IEEE for possible publication. Copyright may be transferred without notice, after which this version may no longer be accessible.}}
\fi

\begin{document}

\maketitle

\begin{abstract}
In cross-silo split learning, Parallel Split Learning with Global Sampling forms
representative pooled batches when class distributions differ across clients,
but ignores client delay when several clients can supply the same class. We introduce Latency Budgeted Parallel Split Learning
with Global Sampling, which separates each pooled batch's integer class target
from the choice of clients that supply its examples. The flow variant formulates this
assignment as an integral network-flow problem and minimizes modeled client-side
completion time for the current target. The fast variant uses a greedy
next-completion rule to reduce schedule-construction cost. Both preserve the
target stream and use every local example once per epoch. A planning rule
selects between the variants while accounting for the cost of constructing both
candidate schedules. On CIFAR-10, the flow variant reduces modeled training
time by 6.75\%, with a 0.30 percentage-point decrease in final accuracy. On
Tiny ImageNet with 20 candidate classes per client, the fast variant reduces
modeled time by 16.87\% and reaches all four validation targets earlier than
the latency-unaware baseline. Across 405 schedule comparisons, the planning
rule stays within 2\% of the lower realized cost in 96.54\% of cases. In our
evaluation, latency-aware provider assignment reduces modeled training time
without changing the prescribed class targets, while the preferred variant
depends on whether assignment savings outweigh schedule-construction overhead.

\end{abstract}


\section{Introduction}
\label{sec:introduction}

Cross-silo learning trains a shared model across organizations with different
class distributions and computing resources while keeping raw data local. Split
learning places part of the model on a server to reduce client-side
computation~\cite{vepakommaSplitLearning2018}, and parallel split learning (PSL)
executes the client models concurrently for a shared server
update~\cite{jeonPSL2020}. Its synchronous steps couple two questions: which
examples should form a representative global batch, and which clients can supply
them efficiently?

Parallel Split Learning with Global Sampling (GPSL) addresses the first
question~\cite{kohankhakiGPSL2025} by sampling across the pooled client
population, which allows variable local contributions and a global batch size
independent of the number of clients. It therefore forms representative global
batches even when client data are non-IID. However, GPSL
is latency-unaware: when several clients hold the same class, its selection does
not exploit differences in predicted completion time.

The same class target can be assigned to clients in different ways, with
different completion times (Fig.~\ref{fig:problem_separation}). This motivates our central question:
can latency-aware assignment reduce training cost while preserving a
representative class target and the full epoch sample set?

\begin{figure}[!t]
  \centering
  \includegraphics[width=0.98\columnwidth]{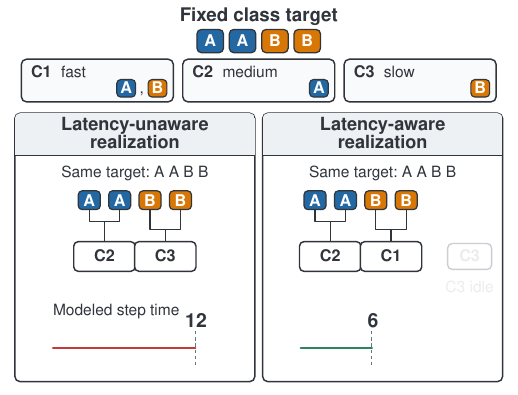}
  \caption{Alternative assignments for the same class target A--A--B--B.
  The latency-unaware assignment uses slow client C3 and takes 12 modeled time
  units. LB-GPSL assigns the same targets to eligible clients C1 and C2 and
  takes 6: the class target is unchanged, only the providers differ.}
  \label{fig:problem_separation}
\end{figure}

We introduce Latency Budgeted Parallel Split Learning with Global Sampling
(LB-GPSL). At each scheduled step, the remaining pooled class counts define an
integer target. Calibrated client delays then guide the assignment of that
target to clients with available examples. A fresh randomized schedule is
constructed before each epoch and leaves the PSL update rule unchanged.
Fig.~\ref{fig:method_overview} summarizes this construction.

Flow LB-GPSL solves an integral flow problem and minimizes modeled
client-side time for the current target and remaining support, at a higher
schedule-construction cost. Fast LB-GPSL uses the earliest predicted next
completion to assign each requested example, reducing construction cost. A
planning rule chooses between them using predicted epoch cost, including
schedule construction.

Our work extends GPSL from latency-unaware global sampling to latency-aware
realization of an explicit class target. The contributions are:
\begin{itemize}
  \item A support-constrained assignment formulation that preserves the
  integer target stream and uses every local example once per epoch. Flow
  LB-GPSL minimizes modeled client-side time for the current target and state,
  and Fast LB-GPSL reaches the same optimum under a sufficient common-support
  condition.
  \item A planning rule that chooses between Flow and Fast LB-GPSL using
  predicted schedule cost, including construction overhead. It requires no
  learning-curve information and stays within 2\% of the lower-cost algorithm
  in 96.54\% of 405 schedule comparisons.
  \item Paired training evidence on CIFAR-10, CIFAR-100, and Tiny ImageNet.
  Flow LB-GPSL reduces modeled training time by 6.75\% on CIFAR-10. Fast
  LB-GPSL reduces it by 16.87\% on Tiny ImageNet with 20 candidate classes per
  client and reaches all four validation targets earlier.
\end{itemize}

We report modeled training time, final accuracy, and time to validation
targets because reassignment can change learning trajectories. On CIFAR-10 and
CIFAR-100, lower modeled training time is accompanied by lower mean final
accuracy.

The remainder of the paper is organized as follows. Section~II reviews related
work. Sections~III--V present the system model, LB-GPSL, and its analysis.
Sections~VI and VII describe the experimental setup and results, followed by
discussion and conclusion in Sections~VIII and IX.

\section{Related Work}
\label{sec:related_work}

\subsection{Split Learning and Global Sampling}

SplitFed combines split learning with federated aggregation of client
models~\cite{thapaSplitFed2022}. Its convergence under heterogeneous data and
dual-paced updates has been studied~\cite{hanSFLConvergence2024}. A review
compares model split, model aggregation, and intermediate aggregation
designs~\cite{huSplitReview2025}.

Several split-learning methods change the learning or aggregation rule. SGLR,
LocFedMix-SL, SMixSL, differentially private split learning, MU-SplitFed, and
class-aware or gradient-aligned variants modify aggregation, representations,
update counts, or learning rules, including privacy--accuracy trade-offs
~\cite{palSGLR2021,ohLocFedMix2022,tinhSMixSL2026,
phamDPSplit2025,liangMUSplitFed2025,xieClassImbalanceSFL2025,linGAPSL2026}.
Other systems focus on execution. EPSL, PipeSFL, HASFL, and SuperSFL adapt
aggregation, model partitioning, parallelism, or client capacity
~\cite{linEPSL2024,gaoPipeSFL2025,linHASFL2026,asifSuperSFL2026}. Wireless and
hybrid systems optimize clustering, spectrum, protocols, helpers, or
schedules~\cite{wuWirelessPSL2023,tiranaHFSL2025}. Our assignment problem keeps
the PSL update, cut, client capacity, and pooled class target fixed.

GPSL~\cite{kohankhakiGPSL2025} forms representative pooled batches from non-IID
data with a global batch size independent of client count. It chooses class
composition and providers together without delay information. LB-GPSL fixes
the class target before assigning its examples to eligible clients under their
remaining support, while still using every local example during the epoch.

\subsection{Heterogeneity, Selection, and Load Balancing}

FedNova shows how unequal update counts alter the effective federated-learning
objective~\cite{wangFedNova2020}. Oort, TiFL, and FedBalancer control
participation, tiers, samples, or pace~\cite{laiOort2021,chaiTiFL2020,shinFedBalancer2022}.
HeteroFL adapts client capacity~\cite{diaoHeteroFL2021}, and NetPlacer+
balances partition placement~\cite{gaoNetPlacer2025}. LB-GPSL instead assigns
a fixed class target among eligible clients while preserving the complete epoch
sample set.

\section{System Model and Problem Formulation}
\label{sec:system_model}

We consider synchronous cross-silo PSL for supervised classification
~\cite{vepakommaSplitLearning2018,jeonPSL2020}. A server coordinates clients
\(\mathcal K=\{1,\ldots,K\}\), and \(\mathcal M=\{1,\ldots,M\}\) denotes the
label set. Client \(k\) stores \(N_k\) examples, including \(n_{k,m}\) examples
of class \(m\). The pooled training population contains
\(N_0=\sum_kN_k\) examples. Clients report their class counts and calibrated
delay profiles to the server. Raw examples remain local.

\subsection{One Synchronous PSL Update}

The network is split into client models \(f_c(\cdot,w_c)\) and a
server model \(f_s(\cdot,w_s)\). Each client stores a synchronized replica
of the client model. At step \(t\), client \(k\) receives a local batch of size
\(b_k^{(t)}\). The active clients are those assigned at least one example,
\begin{equation}
  \mathcal A_t=\{k\in\mathcal K:b_k^{(t)}>0\}.
  \label{eq:active_set}
\end{equation}
They compute smashed activations and send them with their labels to the server.
The server forms one pooled batch of size
\(B_t=\sum_{k\in\mathcal A_t}b_k^{(t)}\), evaluates the mean loss, and
backpropagates through the server model.

The server returns the cut-layer gradients, and each active client completes
its local backward pass. The server then averages the client-model gradients
for that pooled step and applies the same update to every client replica. The
server model is updated once from the pooled loss. This differs from periodic
federated model averaging because the client gradients are averaged at every
PSL step. LB-GPSL uses the same update rule.

\subsection{Class--Client Allocation}

Let \(r_{k,m}^{(t)}\) be the unused class-\(m\) support at client \(k\)
before step \(t\). That client is eligible for the class while
\(r_{k,m}^{(t)}>0\). The number of eligible clients is
\begin{equation}
  \rho_m^{(t)}=\left|\{k:r_{k,m}^{(t)}>0\}\right|.
  \label{eq:client_redundancy}
\end{equation}
This quantity differs from \(C\), the number of candidate classes assigned to
each client. With balanced candidate supports, each class has approximately
\(KC/M\) candidate providers. The realized redundancy \(\rho_m^{(t)}\) counts
only clients that still hold unused examples of class \(m\), so it can decrease
during an epoch.

LB-GPSL represents a pooled batch by class--client counts. The server chooses
integral counts \(q_{k,m}^{(t)}\) with
\begin{equation}
  0\le q_{k,m}^{(t)}\le r_{k,m}^{(t)}, \quad b_k^{(t)}=\sum_mq_{k,m}^{(t)}.
  \label{eq:allocation_counts}
\end{equation}
\begin{figure*}[!t]
  \centering
  \includegraphics[width=0.92\textwidth]{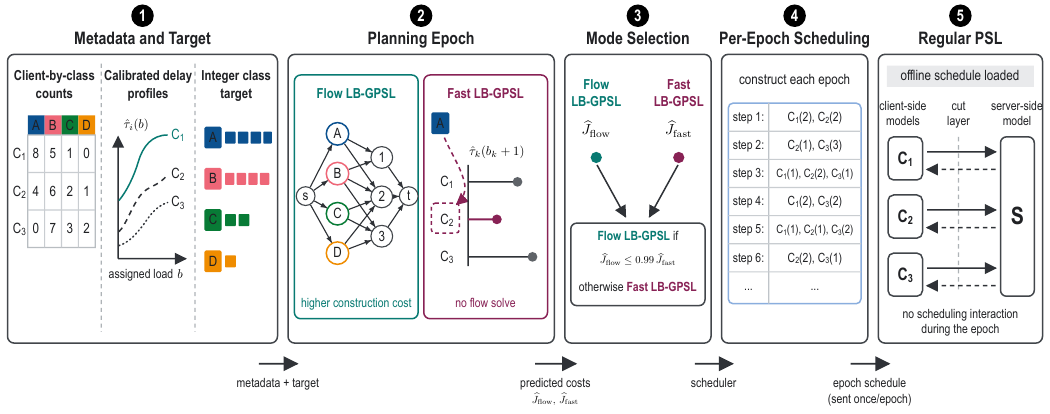}
  \caption{LB-GPSL separates the pooled class target from client assignment.
  During planning, each variant constructs a candidate schedule for one epoch.
  The planning rule compares the two schedules including their construction
  time. The selected variant then constructs a fresh randomized schedule in
  every epoch, and training proceeds with the standard PSL update rule.}
  \label{fig:method_overview}
\end{figure*}

A full pooled batch has \(B\) examples. To include the final partial batch,
the size at step \(t\) is
\begin{equation}
  B_t=\min\!\left\{B,\sum_{k,m}r_{k,m}^{(t)}\right\}.
  \label{eq:partial_batch}
\end{equation}
After assigning that batch, the remaining support becomes
\begin{equation}
  r_{k,m}^{(t+1)}=r_{k,m}^{(t)}-q_{k,m}^{(t)}.
  \label{eq:remaining_update}
\end{equation}

Let \(R_m^{(t)}=\sum_kr_{k,m}^{(t)}\) and
\(R^{(t)}=\sum_mR_m^{(t)}\). The remaining pooled class distribution is
\begin{equation}
  \beta_m^{(t)}=R_m^{(t)}/R^{(t)}.
  \label{eq:remaining_distribution}
\end{equation}
The desired integer class target \(\hat h^{(t)}\) has total mass \(B_t\)
and approximates \(B_t\beta^{(t)}\). An assignment must satisfy
\begin{equation}
  \sum_kq_{k,m}^{(t)}=\hat h_m^{(t)}, \quad m\in\mathcal M.
  \label{eq:exact_target}
\end{equation}
The target deviation and realized class deviation are
\begin{align}
  \Delta_{\mathrm{tgt}}^{(t)}(h)&=\|h/B_t-\beta^{(t)}\|_1,
  \label{eq:target_deviation}\\
  \Delta_{\mathrm{cls}}^{(t)}&=\|B_t^{-1}\sum_kq_{k,:}^{(t)}-\beta^{(t)}\|_1.
  \label{eq:class_deviation}
\end{align}
Realizing \(\hat h^{(t)}\) exactly makes the two quantities equal.

\subsection{Calibrated Client-Side Timing Model}

Let \(\tau_k(b)\) denote the predicted client-side time for client \(k\) with a
local batch of size \(b\). It includes forward computation, activation upload,
gradient download, and backward computation. We use the calibrated model
\begin{equation}
  \tau_k(b)=\mathbf1\{b>0\}(a_k+c_kb^\gamma).
  \label{eq:delay_model}
\end{equation}
The fixed term \(a_k\) captures batch-independent computation and round-trip
costs. The coefficient \(c_k\) combines batch-dependent client computation
and payload transfer, while \(\gamma\) captures the measured batch-size
response. The calibration procedure and held-out validation are reported in
Section~\ref{sec:experimental_setup}.

For an assignment \(q^{(t)}\), the modeled client-side timing objective is
\begin{equation}
  T_{\mathrm{client}}^{(t)}(q)=\max_{k\in\mathcal K}\tau_k\!\left(\sum_mq_{k,m}^{(t)}\right).
  \label{eq:step_time}
\end{equation}
Let the forward--upload and downlink--backward components be
\[
F_k(b)=t_k^{\mathrm{fwd}}(b)+t_k^{\mathrm{up}}(b),
\qquad
D_k(b)=t_k^{\mathrm{down}}(b)+t_k^{\mathrm{bwd}}(b).
\]
Then \mbox{$\tau_k(b)=F_k(b)+D_k(b)$}. Modeling the two synchronization
barriers separately gives
\begin{equation}
  T_{\mathrm{2bar}}^{(t)}=\max_k F_k(b_k^{(t)})+\max_k D_k(b_k^{(t)}).
  \label{eq:two_barrier_time}
\end{equation}
For the same assignment,
\(T_{\mathrm{client}}^{(t)}\le T_{\mathrm{2bar}}^{(t)}\le
2T_{\mathrm{client}}^{(t)}\). The first inequality is an equality exactly when
one client attains both barrier maxima. LB-GPSL optimizes the modeled
client-side objective in~\eqref{eq:step_time}. In the evaluation, we reconstruct
\eqref{eq:two_barrier_time} as a sensitivity check. Server computation is
the same for assignments with the same pooled batch size and is added to
complete modeled-time comparisons.

\subsection{Optimization Problem}

For the current remaining support and target, Flow LB-GPSL solves
\begin{equation}
\begin{aligned}
  \min_{q^{(t)}}\quad &T_{\mathrm{client}}^{(t)}(q)\\
  \mathrm{s.t.}\quad &\sum_kq_{k,m}^{(t)}=\hat h_m^{(t)}, &&m\in\mathcal M,\\
  &0\le q_{k,m}^{(t)}\le r_{k,m}^{(t)}, &&k\in\mathcal K,\ m\in\mathcal M,\\
  &q_{k,m}^{(t)}\in\mathbb Z.&&
\end{aligned}
\label{eq:lb_gpsl_problem}
\end{equation}
The first constraint fixes both the class target and global batch size.
The decision concerns one step under the current remaining support and
calibrated delay model.

\section{Latency Budgeted Parallel Split Learning with Global Sampling}
\label{sec:method}

LB-GPSL first constructs a pooled-batch class target and then assigns its
examples to eligible clients. Flow LB-GPSL minimizes the modeled client-side
timing objective for the current target and support. Fast LB-GPSL assigns each
requested example to the eligible client with the earliest predicted next
completion, reducing schedule-construction cost. Both algorithms use the same
target stream.

\subsection{Integer Target Stream}

Let \(x_m=B_t\beta_m^{(t)}\). Deterministic rounding starts from
\(\lfloor x_m\rfloor\) and assigns the remaining units to the largest
fractional parts. The randomized construction first permutes the classes,
draws \(U\sim\mathrm{Uniform}(0,1)\), and applies systematic rounding:
\begin{equation}
  \hat h_j=\left\lfloor\sum_{i\le j}x_i+U\right\rfloor-\left\lfloor\sum_{i<j}x_i+U\right\rfloor.
  \label{eq:randomized_rounding}
\end{equation}
Mapping back to the original class order, both constructions return
coordinates in \(\{\lfloor x_m\rfloor,\lceil x_m\rceil\}\). Since
\(B_t\le R^{(t)}\) and \(R_m^{(t)}\) is integral,
\(\lceil x_m\rceil\le R_m^{(t)}\). The target therefore never requests more
examples of a class than remain in the pooled support. Fixing the class
permutation and offset at every step fixes the target stream independently of
the assignment algorithm.

\subsection{Flow Assignment}

For a candidate budget \(L\), client \(k\) can receive at most
\begin{equation}
  u_k(L)=\max_{\substack{b\in\mathbb Z_{\ge0}\\
                        b\le R_k^{(t)},\,\tau_k(b)\le L}} b,
  \qquad R_k^{(t)}=\sum_m r_{k,m}^{(t)}.
  \label{eq:capacity}
\end{equation}
Only attainable positive completion times can change this capacity:
\begin{equation}
  \mathcal L^{(t)}=\{a_k+c_kb^\gamma:k\in\mathcal K,\ 1\le b\le\min(B_t,R_k^{(t)})\}.
  \label{eq:candidate_set}
\end{equation}
We remove duplicates and sort the candidates.

At a fixed budget, a source--class--client--sink network tests feasibility.
The source-to-class edge for class \(m\) has capacity \(\hat h_m^{(t)}\). The
class-to-client edge \((m,k)\) has capacity \(r_{k,m}^{(t)}\). The final edge
for client \(k\) has capacity \(u_k(L)\). A flow of value \(B_t\) is exactly a
feasible assignment. Feasibility is monotone in \(L\), so binary search finds
the smallest feasible candidate \(L^\star\).

Several assignments can attain \(L^\star\). To choose among them, the tie-break
favors clients that have contributed less than their proportional share of the
epoch so far. After \(S^{(t)}\) examples have been scheduled, the depletion
deficit is
\begin{equation}
  d_k^{(t)}=\max\{N_kS^{(t)}/N_0-(N_k-R_k^{(t)}),0\}.
  \label{eq:depletion_deficit}
\end{equation}
A minimum-cost feasible flow assigns \(-d_k^{(t)}\) to client \(k\)'s final
edge, resolving equal-budget assignments without changing the target or
budget. Integral capacities give integral counts.

\subsection{Fast Assignment}

Fast LB-GPSL expands the same target into a randomized class sequence. For
each requested class \(m\), it considers clients with unused class-\(m\)
support and chooses the client minimizing
\begin{equation}
  \tau_k(b_k^{(t)}+1).
  \label{eq:fast_next_completion}
\end{equation}
Exact next-completion ties favor the client that has contributed less relative
to its pooled share so far. Remaining ties are broken by client index. Fast LB-GPSL preserves the target and epoch sample set
without candidate generation, binary search, or network flow.

\subsection{Planning-Based Algorithm Selection}

Before training, each algorithm constructs a candidate schedule for one epoch
from the same counts, profile, target stream, and seed. For
\(a\in\{\mathrm{flow},\mathrm{fast}\}\), let
\begin{equation}
  \widehat J_a=S_a^{(0)}+
  \sum_t\max_k\tau_k\!\left(b_{k,t}^{(a,0)}\right),
  \label{eq:mode_selection_cost}
\end{equation}
where \(S_a^{(0)}\) is the measured time to construct the planning schedule. The
rule selects Flow LB-GPSL when
\begin{equation}
  \widehat J_{\mathrm{flow}}\le(1-\varepsilon)\widehat J_{\mathrm{fast}},
  \qquad \varepsilon=0.01,
  \label{eq:mode_selection_rule}
\end{equation}
and selects Fast LB-GPSL otherwise. The 1\% margin selects Flow only when its predicted cost is at least 1\%
lower. Selection depends only on the two predicted schedule costs, which are
available after both planning schedules have been constructed.

When this rule is used for training, the selected planning schedule is used for
the first epoch. The selected algorithm constructs a new randomized schedule
for each later epoch, and the unused planning schedule is discarded. For the
schedule-only comparison, let
\(C_a\) be the measured cost of one complete epoch schedule under the realized
delay profile, including its construction. For an \(E\)-epoch accounting
window, we multiply this measured epoch cost by \(E\) and add the one-time
construction cost of the unselected planning schedule:
\begin{equation}
  C_{\mathrm{sel}}(E)=E C_{a^\star}+S_{\bar a}^{(0)},
  \label{eq:mode_selection_paid_cost}
\end{equation}
where \(\bar a\) is the unselected algorithm. The 100- and 200-epoch comparisons
use the same measured schedule costs with different accounting windows.

\subsection{Offline Epoch Construction}

Algorithm~\ref{alg:lb_gpsl} builds the epoch schedule step by step at the
server while updating the remaining class counts. At the start of each epoch,
the class permutation, systematic-rounding offset, and allocation order are
redrawn. Once the schedule is complete, the server sends it to the clients.
At each scheduled step, clients draw unused examples from the requested
classes, so the class counts are fixed by the schedule but the selected sample
identities remain random. Schedule construction finishes before client
communication begins for the epoch.

\begin{algorithm}[!t]
\caption{Offline construction of one LB-GPSL epoch}
\label{alg:lb_gpsl}
\begin{algorithmic}[1]
\Require Counts \(n\), batch size \(B\), delays \(\tau\), selected mode
\State \(r\gets n\), \(S\gets0\), schedule \(\mathcal Q\gets[\,]\)
\While{\(\sum_{k,m}r_{k,m}>0\)}
  \State \(B_t\gets\min\{B,\sum_{k,m}r_{k,m}\}\)
  \State \(\hat h\gets\Call{SystematicTarget}{r,B_t}\)
  \If{Flow LB-GPSL is selected}
    \State \(L^\star\gets\Call{FirstFeasibleBudget}{r,\hat h,\tau}\)
    \State \(q\gets\Call{MinCostFlow}{r,\hat h,L^\star,S}\)
  \ElsIf{Fast LB-GPSL is selected}
    \State \(q\gets\Call{NextCompletion}{r,\hat h,S,\tau}\)
  \EndIf
  \State verify \(q\) against \(r\) and \(\hat h\)
  \State append \(q\) to \(\mathcal Q\), \(r\gets r-q\), \(S\gets S+B_t\)
\EndWhile
\State \Return \(\mathcal Q\)
\end{algorithmic}
\end{algorithm}

\section{Analysis}
\label{sec:analysis}

We analyze three aspects separately: the class target, the current-step client
assignment, and its effect on the update. Full proofs and counterexamples are in
Supplement Sec.~I.

\subsection{Target and Feasibility Guarantees}

\begin{proposition}[Target accuracy]
\label{prop:target_accuracy}
Both target constructions produce an integer histogram \(\hat h^{(t)}\) of
mass \(B_t\) satisfying
\begin{equation}
  |\hat h_m^{(t)}-B_t\beta_m^{(t)}|<1, \quad
  \Delta_{\mathrm{tgt}}^{(t)}(\hat h^{(t)})<M/B_t.
  \label{eq:target_bound}
\end{equation}
Randomized systematic rounding also satisfies
\begin{equation}
  \mathbb E[\hat h_m^{(t)}\mid\beta^{(t)}]=B_t\beta_m^{(t)}.
  \label{eq:target_unbiased}
\end{equation}
\end{proposition}
\begin{proof}[Proof sketch]
Each coordinate is the floor or ceiling of its real-valued target. Under a
uniform offset, the expected number of systematic grid crossings equals the
interval length \(B_t\beta_m^{(t)}\).
\end{proof}

\begin{proposition}[Invariance of the target stream]
\label{prop:target_stream}
Fix the initial pooled class counts and every random variable used to construct
the integer targets. Any two feasible realization policies then generate the
same complete target stream.
\end{proposition}
\begin{proof}[Proof sketch]
Every feasible realization removes exactly \(\hat h^{(t)}\) from the pooled
counts. The next pooled state is therefore independent of provider identity.
Induction fixes every later pooled state and target. Client-level remaining
states can still differ.
\end{proof}

\begin{proposition}[Flow feasibility and integrality]
\label{prop:flow_feasibility}
For a fixed budget \(L\), target \(\hat h^{(t)}\) is feasible if and only if
the source--class--client--sink network has maximum-flow value \(B_t\).
Equivalently, every \(S\subseteq\mathcal M\) satisfies
\begin{equation}
  \sum_{m\in S}\hat h_m^{(t)}\le
  \sum_{k\in\mathcal K}\min\!\left\{u_k(L),
  \sum_{m\in S}r_{k,m}^{(t)}\right\}.
  \label{eq:hall_condition}
\end{equation}
Whenever feasible, the network and depletion tie-break admit an integral
assignment.
\end{proposition}
\begin{proof}[Proof sketch]
A feasible assignment routes each class demand through eligible clients and
their budget capacities. The reverse mapping converts any value-\(B_t\) flow
into an assignment. The subset condition follows from max-flow/min-cut, and integral capacities
imply integral maximum and minimum-cost flows~\cite{ahujaNetworkFlows1993}.
\end{proof}

\subsection{Current-Step Assignment Guarantees}

\begin{proposition}[Current-step client-side optimality]
\label{prop:min_latency}
Fix \(r^{(t)}\), \(\tau\), and \(\hat h^{(t)}\). Flow LB-GPSL
solves~\eqref{eq:lb_gpsl_problem}. No feasible assignment for the same current state
and target has lower \(T_{\mathrm{client}}^{(t)}\).
\end{proposition}
\begin{proof}[Proof sketch]
The capacities \(u_k(L)\) and flow feasibility are monotone in \(L\). Every
nonempty integral assignment attains a value in \(\mathcal L^{(t)}\). The
search therefore finds the smallest feasible completion time, and the tie-break
stays within that budget.
\end{proof}

The optimum is stepwise: each assignment changes later client--class support, so
the current-step optimum and the full-epoch optimum are different problems.

\begin{proposition}[A sufficient condition for Fast LB-GPSL current-step optimality]
\label{prop:fast_exactness}
For one target, suppose every requested class has the same eligible client set
and no class-specific support count can bind within that target. If every
\(\tau_k(b)\) is nondecreasing, Fast LB-GPSL attains the minimum current-step modeled
client-side time.
\end{proposition}
\begin{proof}[Proof sketch]
Class identity then does not constrain the allocation, so each client
contributes its ordered sequence \(\tau_k(1),\ldots,\tau_k(B_t)\). Fast LB-GPSL
performs a multiway merge that selects the \(B_t\) smallest attainable
completion thresholds, whose maximum is the smallest client-side timing budget
with total capacity \(B_t\).
\end{proof}

Having many providers for each class is not sufficient for this condition: the
requested classes must be available on the same set of clients. The supplement
gives a restricted-support example in which Fast LB-GPSL takes twice the
modeled client-side time of Flow LB-GPSL.

\subsection{Mean Gradients and Epoch Conservation}

At a fixed model state \(w\), let \(\bar g(w,q)\) be the mean of the per-example
gradients sampled under assignment \(q\). Assume that
sampling from client \(k\), class \(m\), has the same expected gradient
\(\mu_{k,m}(w)\) under either assignment, and that
\(\|\mu_{k,m}(w)-\mu_m(w)\|\le\eta_m\). Here, \(\eta_m\) bounds how much the
expected gradient for class \(m\) can differ across clients from the reference
mean \(\mu_m(w)\). For assignments \(q,q'\) with the same class target, define
\begin{equation}
  d_m(q,q')=\frac12\sum_k|q_{k,m}-q'_{k,m}|.
  \label{eq:within_class_transport}
\end{equation}

\begin{proposition}[Bounded within-class mean-gradient shift]
\label{prop:expected_update}
Under the stated assumptions,
\begin{equation}
 \left\|\mathbb E\bar g(w,q)-\mathbb E\bar g(w,q')\right\|
 \le \frac{2}{B_t}\sum_m\eta_m d_m(q,q').
 \label{eq:expected_gradient_bound}
\end{equation}
If \(\eta_m=0\) for every class, all feasible realizations of the target have
the same expected mean gradient.
\end{proposition}
\begin{proof}[Proof sketch]
Both assignments have total class count \(\hat h_m\). Subtract \(\mu_m\) from
the client-class means and apply the triangle inequality to the reassigned
class mass.
\end{proof}

The bound concerns the mean of per-example gradients at one model state. The
actual training update can also depend on the full batch through server-side
batch normalization and on previous gradients through the optimizer. We
therefore evaluate learning behavior empirically through accuracy and time to
target.

Every complete schedule conserves the local class counts.
Summing~\eqref{eq:remaining_update} over an epoch ending at \(r=0\) gives
\begin{equation}
  \sum_tq_{k,m}^{(t)}=n_{k,m} \quad \forall k,m.
  \label{eq:epoch_conservation}
\end{equation}
Because clients draw unused examples without replacement, every local example
is used exactly once per epoch, including in the final partial batch. LB-GPSL
can change the order in which examples are used and which clients supply each
target, but not the set of examples used in the epoch.

Even with the same epoch sample set, different assignments can produce
different learning trajectories. They can therefore require different numbers
of updates to reach a given accuracy, so a lower modeled cost per update does
not necessarily yield an earlier crossing. For algorithm \(a\), let \(N_a(\alpha)\) be the number of updates
required to reach accuracy \(\alpha\), and let \(\bar c_a(\alpha)\) be the
average modeled cost of those updates. Since the modeled time to the crossing is
\(N_a(\alpha)\bar c_a(\alpha)\), LB-GPSL reaches \(\alpha\) earlier than GPSL
exactly when
\begin{equation}
  \frac{N_{\mathrm{LB}}(\alpha)}{N_{\mathrm{GPSL}}(\alpha)}
  <
  \frac{\bar c_{\mathrm{GPSL}}(\alpha)}
       {\bar c_{\mathrm{LB}}(\alpha)}.
  \label{eq:updates_cost_condition}
\end{equation}
We therefore report time to target together with final accuracy and modeled
training time.

\subsection{Scheduling Cost}

The general-flow network has \(n_g=M+K+2\) vertices and at most
\(m_g=MK+M+K\) arcs. At most \(C_t\le KB_t\) completion candidates are
generated. A conservative general-flow bound is
\begin{equation}
  O(C_t\log C_t+\log C_t\,n_g^2m_g+B_tn_gm_g).
  \label{eq:exact_complexity}
\end{equation}
For \(M\le10\), the implementation tests the \(2^M-1\) nonempty subset
conditions in vectorized form. Larger class spaces use general maximum flow.
Fast LB-GPSL avoids the flow solve. Including the work needed to maintain the
class--client counts, a conservative bound is
\(O\bigl(B_tK(M+\log K)+M\bigr)\).

\section{Experimental Setup}
\label{sec:experimental_setup}

We compare learning quality and modeled time against GPSL. A fixed-target
control isolates provider assignment, and schedule-only comparisons test the
planning rule across class supports and delay profiles.

\subsection{Data and Training Protocol}

All training studies use 50 clients. The training examples of each class are
distributed across clients using a Dirichlet distribution with concentration
0.1~\cite{hsuNonIID2019},
producing non-IID class distributions across clients throughout the training
experiments. For CIFAR-10~\cite{krizhevskyCIFAR2009},
we reserve a stratified 5,000-image validation set. All clients
are candidate providers for all ten classes and the remaining 45,000 training
images are partitioned.

The larger-class studies also vary each client's candidate class support.
CIFAR-100~\cite{krizhevskyCIFAR2009} allows all 100 classes at each client, while
Tiny ImageNet~\cite{leTinyImageNet2015} assigns each client either 20 or 100 of
its 200 classes. We call these conditions TIN-C20 and TIN-C100. Candidate
provider counts are balanced before examples are assigned. Because a candidate
class can receive zero examples at a client, the number of clients that
actually hold a class can be smaller than the candidate support suggests.

CIFAR-100 uses 45,000 training and 5,000 validation images. Tiny ImageNet uses
90,000 training and 10,000 validation images. We train a ResNet-18 on CIFAR-10
for 100 epochs. For CIFAR-100 and Tiny ImageNet, we train a ResNet-34 from
random initialization~\cite{heResNet2016} for 100 and 200 epochs, respectively. The server model begins at the second residual-block group for both
CIFAR datasets and at the third for Tiny ImageNet. Client models use 32-group
normalization~\cite{wuGroupNorm2018} and the server model retains batch
normalization. The global batch size is 128.

Training uses AdamW~\cite{loshchilovAdamW2019} with learning rate \(10^{-3}\),
weight decay \(5\times10^{-4}\), five warm-up epochs, cosine decay to
\(10^{-5}\), and label smoothing 0.1. Augmentation uses padded random crops
and horizontal flips for CIFAR, plus random augmentation with two operations
and magnitude nine for Tiny ImageNet. Validation accuracy is recorded after each epoch, and the final
checkpoint is evaluated once on the designated final split: the official CIFAR
test set or Tiny ImageNet's labeled validation set.

\subsection{Compared Methods and Assignment Control}

GPSL is the latency-unaware baseline. The CIFAR-10 study compares GPSL, Flow
LB-GPSL and Fast LB-GPSL over five matched runs, and the larger-class studies
compare GPSL and Fast LB-GPSL over five matched runs. Within a matched pair,
methods share the data partition, initial model, and prescribed randomness.

The target-matched control uses the same integer target stream as LB-GPSL but
assigns examples without using client delays. It uses the depletion tie-break
for otherwise equivalent assignments. The class targets and epoch sample set
therefore remain fixed, so comparisons with Flow or Fast isolate the effect of
latency-aware provider assignment.

Vanilla PSL and sequential split learning provide contextual comparisons with
three runs each under their native update schedules.

\subsection{Timing Model and Cost Accounting}

Client computation is measured at batch sizes 1, 4, 16, and 64, with 2, 8,
32, and 128 held out for validation. The CIFAR-10 ResNet-18 profile uses the
affine form with \(\gamma=1.0\). The fitted ResNet-34 exponents are
\(\gamma=1.355\) for CIFAR-100 and \(1.153\) for Tiny ImageNet.
Maximum held-out percentage errors for the ResNet-34 profiles are 8.872\% and
10.833\%, respectively.
The supplement reports the ResNet-34 profile coefficients and held-out
validation errors.
We form nominal, log-normal, persistent-straggler, and dynamic-jitter profiles
by perturbing the fixed and batch-dependent computation and communication
terms.

Modeled client-side time is obtained by summing Eq.~\eqref{eq:step_time} over
training steps and includes client forward computation, activation upload,
gradient download, and backward computation. Learning curves also include the
measured time required to construct each epoch schedule. For the larger-class
time-to-target comparisons, we additionally include H200 server computation at
the realized global-batch sizes. Because training constructs a new randomized
schedule for every epoch, schedule-construction time is included for every
epoch. As a sensitivity check, we reconstruct the two PSL synchronization
barriers from the calibrated component fits.

\subsection{Schedule-Selection Study}

The schedule-only study varies class count, client count, and provider
redundancy to examine the balance between modeled client-side time and
schedule-construction time. It includes 46 nominal settings: ten based on the
training configurations and 36 additional combinations. Five perturbed profiles
are added for seven of the training-based settings, giving 81 settings in total
and five seeds per setting (405 comparisons). The supplement lists settings and
perturbations.

For each setting and seed, both algorithms build an epoch schedule from
the same counts, targets, and planning profile. The rule selects an algorithm
using costs predicted under that profile. Both schedules are then evaluated
under realized delays. The one-percent margin was fixed on separate development
schedules. We evaluate 100- and 200-epoch accounting windows by scaling the
measured epoch costs and adding the one-time construction cost of the
unselected planning schedule. The training experiments construct a fresh
randomized schedule every epoch.

\subsection{Evaluation Metrics}

Learning curves plot validation accuracy against cumulative modeled time.
Time to target is the first recorded validation point meeting a prespecified
threshold, without interpolation. Runs that do not reach it have no crossing.

We summarize run-level quantities by mean and sample SD. Within each
matched pair we subtract GPSL from LB-GPSL for accuracy and duration, then
report mean differences and descriptive 95\% \(t\)-intervals. Accuracy
differences use percentage points. For pair \(j\), time-to-target reduction is
\(\bigl(T_{\mathrm{GPSL},j}-T_{\mathrm{LB},j}\bigr)/T_{\mathrm{GPSL},j}\).
Positive values mean an earlier LB-GPSL crossing. Complete-training CIFAR-10
reduction averages paired percentages. Larger-class reductions compare
aggregate mean durations.

For algorithm selection, regret is the percentage by which the selected
schedule exceeds the lower realized cost of the two algorithms, including
selection overhead. Selection accuracy is the fraction of comparisons in which
the rule chooses the lower-cost algorithm.

\section{Results}
\label{sec:results}

\subsection{Flow and Fast Favor Different Operating Regimes}

\begin{table*}[!t]
\caption{Final accuracy and modeled training time over five matched runs}
\label{tab:primary_integrated}
\vspace{4pt}
\centering
\footnotesize
\renewcommand{\arraystretch}{1.10}
\begin{tabular*}{\textwidth}{@{\extracolsep{\fill}}llrrrr@{}}
\toprule
Condition & Method & Acc. (\%) \(\uparrow\) & Time (min) \(\downarrow\) & \(\Delta\)Time vs. GPSL (min) \(\downarrow\) & \(\Delta\)Acc. vs. GPSL (pp) \\
\midrule
\multirow{2}{*}{CIFAR-10}
 & GPSL & \(91.23\pm0.22\) & \(32.35\pm2.25\)
 & \multirow{2}{*}{\shortstack{\(-2.17\)\\\([-3.06,-1.28]\)}}
 & \multirow{2}{*}{\shortstack{\(-0.30\)\\\([-0.58,-0.02]\)}} \\
 & Flow LB-GPSL & \(90.92\pm0.30\) & \(30.18\pm2.38\) & & \\
\midrule
\multirow{2}{*}{CIFAR-100}
 & GPSL & \(74.30\pm0.27\) & \(51.99\pm5.42\)
 & \multirow{2}{*}{\shortstack{\(-6.91\)\\\([-11.12,-2.70]\)}}
 & \multirow{2}{*}{\shortstack{\(-0.59\)\\\([-1.19,+0.01]\)}} \\
 & Fast LB-GPSL & \(73.71\pm0.35\) & \(45.08\pm8.61\) & & \\
\midrule
\multirow{2}{*}{TIN-C20}
 & GPSL & \(62.41\pm0.23\) & \(313.02\pm25.93\)
 & \multirow{2}{*}{\shortstack{\(-52.81\)\\\([-57.04,-48.58]\)}}
 & \multirow{2}{*}{\shortstack{\(+0.17\)\\\([-0.23,+0.58]\)}} \\
 & Fast LB-GPSL & \(62.58\pm0.30\) & \(260.21\pm25.70\) & & \\
\bottomrule
\end{tabular*}
\vspace{1pt}

{\footnotesize Acc. is final-split accuracy. Time includes modeled client-side time and measured schedule construction, excluding server work. Values are means \(\pm\) sample standard deviations. Brackets are descriptive paired 95\% \(t\)-intervals.}
\end{table*}

Flow LB-GPSL reduces CIFAR-10 modeled training time in all five runs
(Table~\ref{tab:primary_integrated}), with a mean paired reduction of 6.75\%
and a 0.30 percentage-point decrease in final accuracy. In the ten-class setting Fast LB-GPSL provides
little time benefit despite its cheaper construction
(Fig.~\ref{fig:learning_evidence}(a)).

\begin{figure*}[!t]
  \centering
  \includegraphics[width=0.92\textwidth]{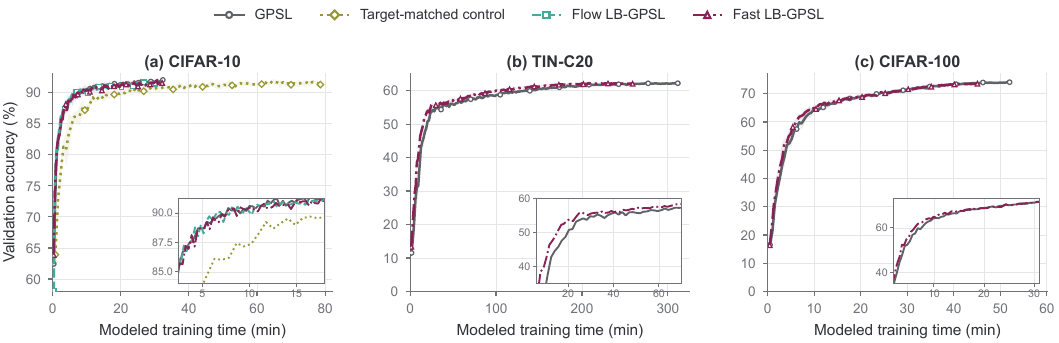}
  \caption{Validation-accuracy learning curves for (a) CIFAR-10,
  (b) TIN-C20, and (c) CIFAR-100. Lines and bands show means and sample
  standard deviations over five runs. The horizontal axis includes
  modeled client-side time plus measured schedule-construction time.
  Panel (a) also includes the target-matched control.}
  \label{fig:learning_evidence}
\end{figure*}

On TIN-C20, Fast LB-GPSL reduces modeled training time by 16.87\% and
reaches 62.58\% final accuracy, compared with
62.41\% for GPSL. On CIFAR-100, it reduces modeled training time by 13.29\%
while final accuracy is 0.59 percentage points lower. The paired intervals
span zero for both conditions, so these results do not establish an accuracy
difference.
An additional TIN-C100 complete-training comparison is reported in Supplement
Sec.~IV.

\subsection{Time-to-Target Gains Vary by Threshold and Class Scale}

The paired time-to-target reductions in Fig.~\ref{fig:threshold_evidence}
vary by accuracy threshold and class scale.
\begin{figure*}[!t]
  \centering
  \includegraphics[width=0.92\textwidth]{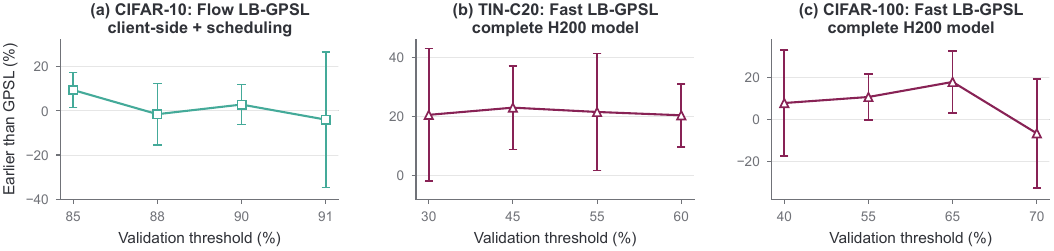}
  \caption{Reduction in modeled time to validation-accuracy targets.
  Positive values indicate an earlier crossing than GPSL.
  Panel (a) includes modeled client-side time plus schedule-construction time.
  Panels (b) and (c) additionally include H200 server computation.
  Markers and bars show paired means and descriptive 95\% \(t\)-intervals
  over five runs.}
  \label{fig:threshold_evidence}
\end{figure*}

Flow LB-GPSL reaches
85\% CIFAR-10 validation accuracy 9.38\% earlier than GPSL, and the later
threshold intervals span zero, so the gain is concentrated at the early target.
Fast LB-GPSL reaches all four TIN-C20 thresholds earlier (20.48--23.06\% after
including server work). On CIFAR-100 it reaches targets through 65\% accuracy
7.93--17.99\% earlier, but reaches 70\% accuracy 6.55\% later. This
reversal shows that lower modeled update cost need not yield an earlier crossing.

\subsection{Assignment Savings Must Offset Construction Cost}

\begin{table}[!t]
\caption{Fixed-target assignment time under nominal profiles}
\label{tab:realization_mechanism}
\vspace{4pt}
\centering
\scriptsize
\renewcommand{\arraystretch}{1.10}
\setlength{\tabcolsep}{3.0pt}
\begin{tabular*}{\columnwidth}{@{\extracolsep{\fill}}lrrrr@{}}
\toprule
Assignment & Time (s) \(\downarrow\) & Build (s) \(\downarrow\) & Build share \(\downarrow\) & Gain \(\uparrow\) \\
\midrule
\multicolumn{5}{c}{CIFAR-100} \\
\midrule
Target-matched control & 100.76 & 4.59 & 4.55 & -- \\
Flow LB-GPSL & 35.99 & 14.03 & 38.97 & 63.91 \\
Fast LB-GPSL & \textbf{22.88} & \textbf{0.57} & \textbf{2.48} & \textbf{77.13} \\
\midrule
\multicolumn{5}{c}{TIN-C20} \\
\midrule
Target-matched control & 76.96 & 8.87 & 11.53 & -- \\
Flow LB-GPSL & 75.94 & 24.63 & 32.43 & 1.22 \\
Fast LB-GPSL & \textbf{60.29} & \textbf{1.08} & \textbf{1.80} & \textbf{21.94} \\
\bottomrule
\end{tabular*}
\vspace{1pt}

{\footnotesize Time is modeled client-side time plus schedule construction, which is listed separately as Build. Build share and gain vs. control are percentages.}
\end{table}

The target-matched comparison holds every class target fixed and varies only the
client assignment (Table~\ref{tab:realization_mechanism}). Flow LB-GPSL
minimizes modeled client-side time at each step, but its schedule is more
expensive to construct. Relative to the control, total modeled epoch time falls by 63.91\%
for Flow and 77.13\% for Fast on CIFAR-100. On TIN-C20, the reductions are
1.22\% and 21.94\%, respectively. Schedule construction accounts for 38.97\%
and 32.43\% of Flow's epoch time in these two settings, compared with 2.48\%
and 1.80\% for Fast.

\subsection{Planning Usually Selects the Lower-Cost Algorithm}

\begin{figure*}[!t]
  \centering
  \includegraphics[width=0.90\textwidth]{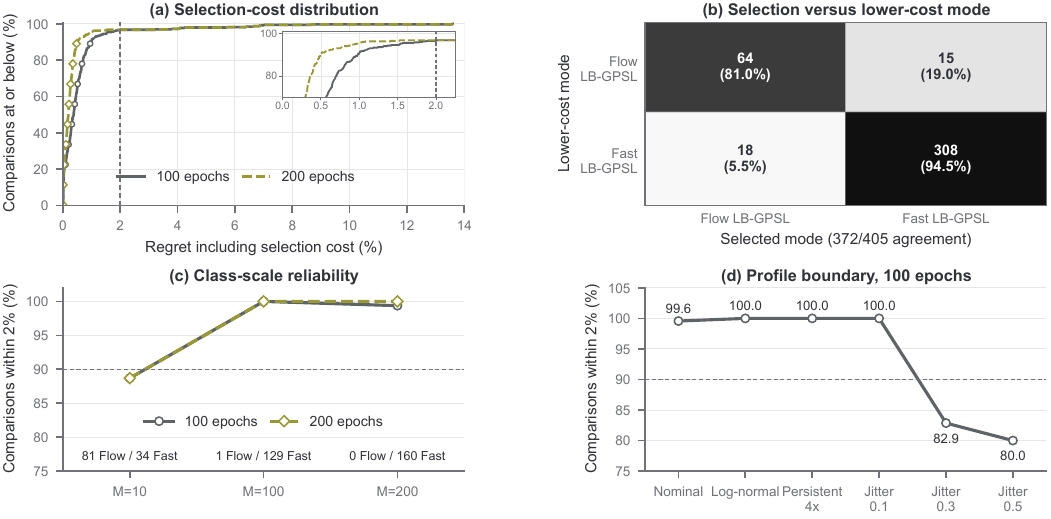}
  \caption{Planning-rule evaluation over 81 settings and five seeds per
  setting. Panel (a) shows the distribution of regret, including selection
  overhead, for 100- and 200-epoch accounting windows.
  Panel (b) compares the selected algorithm with the lower-cost algorithm
  under the realized profile. Panels (c) and (d) show the fraction of
  comparisons within 2\% of that cost by class count and delay profile.
  Flow and Fast denote the two LB-GPSL assignment algorithms.}
  \label{fig:selector_evidence}
\end{figure*}

Over 100 epochs, the planning rule has mean regret 0.596\%, 95th-percentile
regret 1.529\%, and maximum regret 13.588\%
(Fig.~\ref{fig:selector_evidence}(a)). It stays within 2\% of the lower
realized cost in 96.54\% of comparisons and selects the lower-cost algorithm in
91.85\%. The 200-epoch accounting window uses the same decisions and schedules.
Because the unused planning schedule is a one-time cost, its contribution is
smaller over the longer window, reducing mean regret to 0.412\%.

Fast LB-GPSL is selected in 323 of 405 comparisons, including every 200-class
setting. Over 100 epochs, 99.57\% of nominal-profile comparisons and all
log-normal, persistent-straggler, and mild-jitter comparisons are within 2\% of
the lower realized cost. That fraction falls to 82.86\% and 80.00\% for
unmodeled jitter magnitudes 0.3 and 0.5.

\subsection{Contextual Methods Reach Different Accuracy Ranges}

\begin{table}[!t]
\caption{Final accuracy and complete modeled H200 time to validation targets}
\label{tab:context_crossings}
\vspace{4pt}
\centering
\scriptsize
\renewcommand{\arraystretch}{1.10}
\begin{tabular*}{\columnwidth}{@{\extracolsep{\fill}}lrrrrr@{}}
\toprule
\multicolumn{6}{c}{CIFAR-100} \\
\cmidrule(lr){1-6}
Method & Acc. (\%) \(\uparrow\) & \multicolumn{4}{c}{Time to validation target (min) \(\downarrow\)} \\
\cmidrule(lr){3-6}
 & & 40\% & 55\% & 65\% & 70\% \\
\midrule
GPSL & \textbf{74.30} & 3.08 & 6.07 & 12.61 & \textbf{24.69} \\
Fast LB-GPSL & 73.71 & \textbf{2.85} & \textbf{5.42} & \textbf{10.35} & 26.14 \\
PSL & 66.51 & 16.73 & 23.92 & 34.04 & - \\
Sequential SL & 55.40 & 934.00 & 1469.72 & - & - \\
\midrule
\multicolumn{6}{c}{TIN-C20} \\
\cmidrule(lr){1-6}
Method & Acc. (\%) \(\uparrow\) & \multicolumn{4}{c}{Time to validation target (min) \(\downarrow\)} \\
\cmidrule(lr){3-6}
 & & 30\% & 45\% & 55\% & 60\% \\
\midrule
GPSL & 62.41 & 9.39 & 17.44 & 35.10 & 142.97 \\
Fast LB-GPSL & \textbf{62.58} & \textbf{7.47} & \textbf{13.32} & \textbf{27.01} & \textbf{112.86} \\
PSL & 52.34 & 138.58 & 215.44 & - & - \\
Sequential SL & 1.36 & - & - & - & - \\
\bottomrule
\end{tabular*}
\vspace{1pt}

{\footnotesize Times are in minutes and use reaching runs only. A hyphen denotes no crossing. Bold marks the best value under the indicated direction.}
\end{table}

GPSL and Fast LB-GPSL reach all displayed validation targets under their
native schedules (Table~\ref{tab:context_crossings}). Vanilla PSL reaches
only the lower thresholds, while sequential split learning reaches fewer
targets and takes substantially longer on CIFAR-100.
A shorter complete run can therefore reflect lower final accuracy rather
than faster learning.

\Needspace{5\baselineskip}
\section{Discussion and Limitations}
\label{sec:discussion}

For a fixed class target, the reduction in client-side time must be large
enough to offset schedule-construction cost. In the tested larger-class settings, Fast LB-GPSL's lower
construction cost outweighs its extra modeled client-side time. The planning
rule chooses between the algorithms before training.

Clients can hold different examples of the same class, so reassignment can
change the gradients and learning trajectory even when class targets match.
Proposition~\ref{prop:expected_update} bounds the expected mean-gradient shift
for additive per-example gradients at one model state. The lower observed mean
final accuracies on CIFAR-10 and CIFAR-100 show why learning quality must be
assessed alongside modeled time.

The exactness guarantee for Flow LB-GPSL applies to one step: for the current
target and remaining support, it minimizes the modeled client-side objective in
Eq.~\eqref{eq:step_time}. Full-epoch cost and the physical two-barrier PSL step
are evaluated separately. Reconstructing the barriers from calibrated component
fits changes epoch times by less than 0.012\% across the evaluated settings and
leaves the method ordering unchanged. In these settings, the client-side timing closely matches the
two-barrier reconstruction.

The protocol constructs a fresh randomized schedule every epoch. Reusing a
fixed schedule would reduce construction time but repeat the same targets and
providers. Its learning behavior remains untested.

Planning errors grow when realized delays differ from the planning
profile. Online profile updates remain untested. Reported times combine measured schedule construction with calibrated
client delays. They are modeled rather than end-to-end distributed measurements.

Training spans three image datasets, two ResNet architectures, fixed
partitions, and 50 clients. The schedule study weights settings equally, not
by deployment prevalence. Deployment requires class counts and delay
profiles. Privacy-preserving estimation and broader model/network validation
remain open.
\section{Conclusion}
\label{sec:conclusion}

LB-GPSL separates the class composition of a global batch from the clients
that supply it. This allows delay-aware assignment while preserving the
integer target stream, global batch size, and complete epoch sample set.
Flow LB-GPSL minimizes modeled client-side time for each current target.
Fast LB-GPSL builds schedules more cheaply, and a planning rule compares the
variants with construction overhead included.

On CIFAR-10, Flow LB-GPSL reduces modeled training time by 6.75\% with a
0.30 percentage-point decrease in final accuracy. On TIN-C20, Fast LB-GPSL
reduces modeled time by 16.87\% and reaches all four validation targets
earlier. The planning rule usually selects the lower-cost algorithm, with its
largest errors occurring under strong unmodeled jitter.

\section*{Acknowledgment}

Funded by the Deutsche Forschungsgemeinschaft (DFG, German Research Foundation) under Germany’s Excellence Strategy – EXC 3115 – 533767731.

OpenAI ChatGPT and Codex assisted with drafting sections, language editing, and with research code development~\cite{openaiChatGPT2026,openaiCodex2026}.
The authors are responsible for the scientific content and final text.

\bibliographystyle{IEEEtran}
\bibliography{references}

\end{document}